\documentclass[11pt,a4paper]{article}
\usepackage[margin=25mm]{geometry}
\usepackage[T1]{fontenc}
\usepackage{lmodern,microtype,xurl}
\usepackage{amsmath,amssymb,amsthm,mathtools}
\usepackage{graphicx,booktabs,array,enumitem,float}
\usepackage[numbers,sort&compress]{natbib}
\usepackage[dvipsnames]{xcolor}
\usepackage[colorlinks=true,linkcolor=MidnightBlue,citecolor=MidnightBlue,urlcolor=MidnightBlue]{hyperref}
\usepackage[font=small,labelfont=bf]{caption}
\usepackage{fancyhdr}
\setlist{nosep,leftmargin=*}
\newtheorem{theorem}{Theorem}
\newtheorem{proposition}[theorem]{Proposition}
\newtheorem{corollary}[theorem]{Corollary}
\newtheorem{lemma}[theorem]{Lemma}
\newcommand{\ind}{\mathbf{1}}
\newcommand{\R}{\mathbb{R}}
\newcommand{\U}{U_\lambda}

\title{\vspace{-8mm}\textbf{Exact Memory-Time Optimization for\\ Prefix-Cached Language Model Serving}}
\author{Shivam Gupta\\\small Independent Researcher, Delhi, India\\\small\texttt{shivam1720406@gmail.com}}
\date{23 September 2026}
\begin{document}
\maketitle
\begin{abstract}
Retaining language-model prefix states trades recomputation against storage time.
Optimizing each cached block independently can overcount savings: a resident block
is usable only when the required preceding prefix is also available. We introduce
Prefix-Certificate Retention (PCR), an exact finite-trace formulation for static,
grouped, reset-on-access timeouts. Usable-prefix rewards become nodes whose
prerequisites are timeout thresholds and preceding hit certificates. The resulting
maximum-weight closure reduces to one minimum cut, with graph size linear in the
number of block lookups and timeout choices. A breakpoint theorem extends the
construction to all nonnegative timeouts without discretization error. We also
derive a linear-time-in-grid-size dynamic program for ordered timeouts and bounds
that certify the cost of this restriction. Exhaustive small-instance checks and
chronological replay of 39,632 public Mooncake requests validate the formulation.
On the fixed grid, ordered timeouts attain the unrestricted training optimum in
118 of 120 trace--grouping--price cases. Heterogeneous retention improves several
held-out memory-time tradeoffs, but finer training optimization does not uniformly
improve transfer. The contribution is a tractable optimization model and an
auditable benchmark for retention policies; the experiments measure usable prefix
blocks and storage time, not GPU latency.
\end{abstract}
\textbf{Code and reproducibility:}\\
\href{https://github.com/shi1720/prefix-certificate-retention}{\texttt{github.com/shi1720/prefix-certificate-retention}}

\section{Introduction}
Language-model services repeatedly process shared system prompts, conversation
histories and agent instructions. Retaining their key--value (KV) states can avoid
repeated prefill work, but every retained block occupies storage until it is reused
or discarded. PagedAttention, RadixAttention and disaggregated cache architectures
provide practical mechanisms for this reuse \cite{paged,sglang,mooncake}. They also
create an operational decision: how long should different parts of a prefix remain
available?

The optimization is subtle even before queueing and hardware constraints enter.
Consider a request for blocks $(A,B)$. If $B$ remains in an independently addressed
store but $A$ has expired, counting $B$ as a saved prefix block is incorrect for a
serving interface that resumes only from a contiguous cached prefix. Requiring
the timeout of $A$ to be at least that of $B$ prevents this inconsistency. However,
that requirement can retain $A$ unnecessarily: other requests may refresh $A$
while $B$ waits for reuse. Thus, independent hit accounting can be optimistic,
whereas universally prefix-preserving timeout ordering can be conservative.

This paper asks three questions. Can actual prefix reuse be optimized exactly
without imposing timeout order? Can the gap between ordered and unrestricted
retention be certified? And do the resulting profiles improve memory-time
tradeoffs on later requests? The first two questions admit deterministic answers
within a precise model. The third requires empirical evaluation and has a mixed
answer.

We study static group timeouts fitted to a historical trace. The objective is the
number of reusable full prefix blocks minus a declared storage price times
occupied block-seconds. It is a configuration objective for a cache store, not a
complete inference-system objective. The main contributions are:
\begin{enumerate}
\item An exact representation of contiguous-prefix reuse by a directed acyclic
graph of hit certificates. Combining this representation with incremental
storage costs yields a maximum-weight closure problem, solved by a standard
minimum-cut algorithm (Theorem~\ref{thm:cut}).
\item A reuse-age breakpoint theorem that removes the timeout-grid restriction
while retaining a graph of linear size (Theorem~\ref{thm:event}).
\item A coherent-policy dynamic program and a three-way optimality bound,
together with a strict counterexample and expected-budget certificates
(Section~\ref{sec:bounds}).
\item Reproducible tests and public-trace experiments separating mathematical
optimality, the practical value of a simpler restriction, and held-out transfer
(Sections~\ref{sec:exp}--\ref{sec:results}).
\end{enumerate}
Maximum closure and its reduction to minimum cut are classical \cite{picard}.
Our contribution is the prefix-retention formulation, its exact accounting and
breakpoint structure, and the resulting certification study. Neither differentiated
timeouts nor prefix-aware cache management is claimed as new.

\section{Relation to prior work}
\paragraph{Timer-based cache optimization.}
Ferragut, Rodriguez and Paganini characterize timer policies through arrival
distributions and hazard rates \cite{ferragut}. Soft-TTL optimizes time-varying
fractions of retained files \cite{softttl}. These results establish substantial
precedent for retention optimization. Here the input is a finite trace, the policy
is a static timeout per block group, and the reward includes conjunctions of
prefix availability. We do not require a renewal model or estimate hazard rates.

\paragraph{Prefix and agent-serving systems.}
PagedAttention addresses KV memory management \cite{paged}; SGLang's
RadixAttention exploits shared prefixes \cite{sglang}. CachedAttention retains
conversation state across requests \cite{cachedattention}, while Mooncake
integrates a shared KV store with disaggregated serving \cite{mooncake}. Recent
agent-serving work includes workflow-level scheduling in SAGA \cite{saga} and
execution-transition-guided eviction and prefetching in CacheScout
\cite{cachescout}. PCR supplies an offline optimization benchmark for one
retention subproblem. It does not replace their scheduling, transfer or execution
models.

\paragraph{The closest configuration formulation.}
Kareto groups frequent prefix subtrees, models timeout-dependent block hits and
storage exposure, and uses multi-start sequential quadratic programming for a
budgeted group-timeout allocation \cite{kareto}. Grouping by popular prefixes and
fitting timeouts from historical accesses are therefore explicit antecedents.
PCR differs by representing usable-prefix conjunctions directly, censoring
occupancy at the observation horizon, and solving a priced static-timeout problem
globally. Our group baseline is a controlled simplification inspired by this
literature, not a reproduction of Kareto's tiered-storage system or its reported
latency results.

\paragraph{Dependency constraints and graph optimization.}
Dependency-aware online caching studies cache decisions with object prerequisites
and hard-capacity constraints \cite{dependency}. Our unrestricted model permits
orphaned stored blocks but awards reuse only to contiguous prefixes; its constraint
is priced memory-time. This difference is what allows temporal refreshes to make
unordered timeouts useful. Picard's maximum-closure reduction \cite{picard}
provides the algorithmic foundation once the correct hit prerequisites have been
identified. The distinction is in the cache model, not in a new general-purpose
flow algorithm. A dated literature matrix and search scope accompany the code.

\section{Trace model and exact accounting}\label{sec:model}
Let the observation window be $[0,T]$. A request $r$ arrives at time $t_r$ and
contains an ordered list of full blocks $(b_{r,1},\ldots,b_{r,L_r})$. A block ID
identifies the entire prefix ending at that block; each non-root block has a
unique predecessor. The total number of block lookups is
$M=\sum_r L_r$. Blocks are mapped to $G$ fixed groups by $g(b)$. One timeout
$\tau_g\geq0$ is used for every block in group $g$.

The store starts empty. All requests at a common timestamp read the pre-batch
state before any of them writes. After those reads, every full block in the batch
is available and its expiration is reset to $t+\tau_{g(b)}$. Zero timeout means
no retention. A positive timeout remains valid at its expiration instant. These
conventions make tied arrivals and exact boundary hits unambiguous. Physical
storage is keyed by cumulative hashes and may contain a block whose predecessor
is absent. The serving interface counts only the longest contiguous available
prefix; it does not reuse an isolated suffix after recomputing a missing ancestor.

Let $a_{r,d}$ be the time since the most recent \emph{strictly earlier timestamp
batch} containing $b_{r,d}$, with $a_{r,d}=\infty$ on its first access. Then local
availability and usable-prefix availability are
\begin{align}
 x_{r,d}(\tau)&=\ind\{0<a_{r,d}\leq\tau_{g(b_{r,d})}\},\label{eq:local}\\
 z_{r,d}(\tau)&=\prod_{j=1}^{d}x_{r,j}(\tau),\qquad
 H(\tau)=\sum_r\sum_{d=1}^{L_r}z_{r,d}(\tau).\label{eq:hit}
\end{align}
The strict inequality in~\eqref{eq:local} follows from batch semantics; it does
not discard reuse across distinct positive-age batches. More generally, each
$z_{r,d}$ can have an exogenous nonnegative reward $w_{r,d}$. All reported
experiments use $w_{r,d}=1$.

For a physical block $b$ with distinct access timestamps
$s_{b,1}<\cdots<s_{b,m_b}$, set $s_{b,m_b+1}=T$. Its occupied time is
\begin{equation}
 C_b(\tau)=\sum_{j=1}^{m_b}\min\{\tau,s_{b,j+1}-s_{b,j}\}.
 \label{eq:cost}
\end{equation}
Define $C_g(\tau)=\sum_{b:g(b)=g}C_b(\tau)$ and
$C(\boldsymbol\tau)=\sum_g C_g(\tau_g)$. The priced objective is
\begin{equation}
 \max_{\boldsymbol\tau\in\mathcal T}\quad
 \U(\boldsymbol\tau):=H(\boldsymbol\tau)-\lambda C(\boldsymbol\tau),
 \qquad \lambda\geq0.\label{eq:objective}
\end{equation}
The price has units of saved blocks per occupied block-second. Mean occupancy
is $C/T$. For a fixed model whose full blocks occupy $q$ bytes, byte-seconds are
$qC$; converting rewards to money or latency requires additional measurements.

\begin{lemma}[Exact occupied time]\label{lem:cost}
Equation~\eqref{eq:cost} equals the integral of the block's residency indicator
over $[0,T]$, including when accesses have multiplicity within a batch.
\end{lemma}
\begin{proof}
Partition the window after the first access into
$[s_{b,j},s_{b,j+1})$. On each interval there is no intervening reset. The block
is resident for exactly the smaller of its timeout and the interval length.
Summing these disjoint contributions gives~\eqref{eq:cost}. A batch writes one
physical copy regardless of how many requests read it. The final interval ends
at $T$, so no storage outside the observation window is charged.
\end{proof}

The assumptions are consequential. New blocks are available immediately after
their arrival batch, so this is a zero-service-time cache replay. A real system
must wait for computation and transfer. Cache pinning by active requests, decode
state, admission control, model-version invalidation and capacity-triggered
eviction are outside~\eqref{eq:objective}. Changing those semantics can invalidate
both the additive cost and the threshold representation.

\section{Prefix certificates and the minimum-cut formulation}\label{sec:cut}
For group $g$, let its finite ordered timeout choices be
$0=q_{g,0}<q_{g,1}<\cdots<q_{g,K_g}$. A shared grid is a special case. Introduce
threshold variables $y_{g,k}=\ind\{\tau_g\geq q_{g,k}\}$ for $k\geq1$, with
prerequisites $y_{g,k}\Rightarrow y_{g,k-1}$ for $k>1$. The incremental costs
\begin{equation}
 \Delta C_{g,k}=C_g(q_{g,k})-C_g(q_{g,k-1})\geq0
\end{equation}
give $C_g(\tau_g)=\sum_{k=1}^{K_g}\Delta C_{g,k}y_{g,k}$.

For a finite age $a_{r,d}$, define
$\kappa(r,d)=\min\{k\geq1:q_{g(b_{r,d}),k}\geq a_{r,d}\}$.
If this set is empty, or any earlier block in the request has no feasible local
hit, that depth and all later depths cannot contribute and are omitted.
Otherwise introduce a hit node $z_{r,d}$ with implications
\begin{equation}
 z_{r,d}\Rightarrow y_{g(b_{r,d}),\kappa(r,d)},\qquad
 z_{r,d}\Rightarrow z_{r,d-1}\quad(d>1).\label{eq:arcs}
\end{equation}
Give hit nodes weight $w_{r,d}\geq0$ and threshold nodes weight
$-\lambda\Delta C_{g,k}$. A \emph{closed} node set contains every prerequisite
of every node it contains. Choosing a maximum-weight closed set exactly balances
the rewards against the timeouts required to make them usable.

\begin{theorem}[Exact grid optimization]\label{thm:cut}
Under the model of Section~\ref{sec:model}, maximizing~\eqref{eq:objective} over
the finite product grid is a maximum-weight closure problem with
$O(M+\sum_g K_g)$ vertices and arcs. It is solved by one minimum $s$--$t$ cut.
For unit rewards, its optimal value equals $W-F$, where $W$ is the sum of all
hit-node weights and $F$ is the minimum-cut value.
\end{theorem}
\begin{proof}
Any timeout profile selects an initial segment of each threshold chain. Adding
exactly the usable hit nodes yields a closed set: the local threshold and all
earlier hit certificates exist whenever a prefix is usable. The node-weight sum
is its utility by telescoping the incremental costs.

Conversely, a closed set selects initial threshold segments and thus a valid
timeout profile. It cannot contain an unusable hit node, because recursively
following~\eqref{eq:arcs} enforces every required local threshold. Adding any
omitted usable hit node and its hit prerequisites does not require new thresholds
and cannot reduce weight. Hence a maximum-weight closed set attains exactly the
best timeout utility.

For completeness, add an arc from the source to every positive-weight node with
capacity equal to its weight, and an arc from every negative-weight node to the
sink with capacity equal to the negative of its weight. Every prerequisite arc
gets capacity $Q>W$. The all-sink cut has value $W$, so a minimum cut cannot sever
a prerequisite arc. Its source side is therefore closed. Its cut capacity is
$W$ minus the source-side node weight, proving the identity. There is one
threshold node per positive choice, at most one hit node per lookup, and at most
two prerequisite arcs per hit node. This is the classical closure reduction
\cite{picard} applied to the constructed cache graph.
\end{proof}

The graph-size claim is not a linear-time bound for maximum flow. Building
threshold indices by binary search takes
$O(M\log(1+\max_gK_g))$ time once the access ages and grids are available.
Identical certificates can be merged: a certificate is keyed by its predecessor
certificate, group, and threshold index. Summing their rewards preserves every
profile's utility. Our implementation uses this compression.

\begin{figure}[t]
 \centering\includegraphics[width=\linewidth]{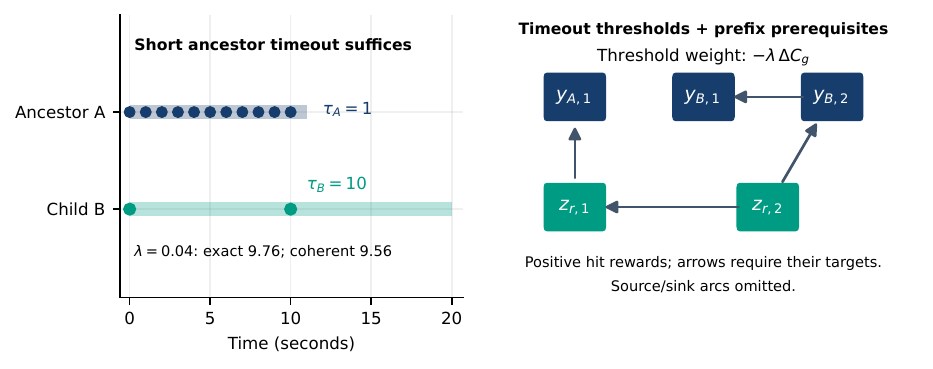}
 \caption{Why prefix certificates are needed. Left: requests for $(A,B)$ at
 times 0 and 10, and for $A$ at times 1 through 9. A short timeout for $A$ can
 coexist with a longer timeout for $B$ because $A$ is refreshed. Right: each
 reward requires its local timeout threshold and the preceding prefix reward.
 Arrows denote prerequisites. The graph schematic is illustrative; the strict
 numerical example is evaluated at horizon 20.}
 \label{fig:model}
\end{figure}

\subsection{Removing the discretization restriction}
Let $\mathcal A_g$ contain all finite positive reuse ages of blocks in group $g$
on the trace, even when the corresponding full prefix cannot be reused. Set
$\mathcal Q_g=\{0\}\cup\mathcal A_g$.

\begin{theorem}[Reuse-age sufficiency]\label{thm:event}
For every $\boldsymbol\tau\in\R_+^G$, there exists
$\boldsymbol\tau'\in\prod_g\mathcal Q_g$ with
$H(\boldsymbol\tau')=H(\boldsymbol\tau)$ and
$C(\boldsymbol\tau')\leq C(\boldsymbol\tau)$. Consequently, for every
$\lambda\geq0$, the continuous nonnegative-timeout problem has an optimum on
these finite grids. Its closure graph has $O(M+G)$ vertices and arcs.
\end{theorem}
\begin{proof}
For each group, round its timeout down to the largest member of $\mathcal Q_g$
not exceeding it. Every reuse-age comparison has the same truth value before
and after rounding: if an age was at most the old timeout, that age itself is
an eligible grid member and is no larger than the rounded timeout. False
comparisons remain false. All local indicators, and hence their prefix products,
are unchanged. Each $C_g$ is nondecreasing, so cost cannot increase. There are at
most $M$ finite reuse ages before deduplication. The graph bound follows from
Theorem~\ref{thm:cut}.
\end{proof}

Cost tables on these grids need not be formed by a quadratic scan. For a group,
collect the exposure lengths in~\eqref{eq:cost}, sort them as
$e_1\leq\cdots\leq e_n$, and store prefix sums $S_j=\sum_{i\leq j}e_i$.
If $j(q)$ is the number of exposures strictly below $q$, then
\begin{equation}
 C_g(q)=S_{j(q)}+(n-j(q))q.\label{eq:fastcost}
\end{equation}
Sorting ages and exposures gives $O(M\log M+G)$ preprocessing and linear graph
storage. The experiment implementation uses integer microseconds for threshold
comparisons and evaluates~\eqref{eq:fastcost} in block-seconds.

\section{Coherent policies, bounds and budgets}\label{sec:bounds}
Suppose the group relations form a forest. If a physical block's predecessor is
in a different group, its group is the parent of the block's group; groups may
also contain consecutive blocks. Write $p(g)$ for a group's parent. A
\emph{coherent} profile satisfies $\tau_g\leq\tau_{p(g)}$ whenever a parent exists.

\begin{proposition}[A sufficient invariant, not a necessary trace condition]
Coherent timeouts ensure that every resident block has its whole preceding
prefix resident, for every trace satisfying the model. Without this ordering,
the invariant need not hold. However, ordering can be strictly suboptimal for
the finite-trace utility.
\end{proposition}
\begin{proof}
Every child access is also an access to its predecessor. At any time, the most
recent predecessor reset is at least as recent as the child's. Its timeout is
also at least as long, so child residency implies predecessor residency;
induction gives the full prefix. If a child's timeout exceeds its predecessor's,
one access followed by silence produces an interval with an orphaned child.

For strict suboptimality, use the trace in Figure~\ref{fig:model}, horizon
$T=20$, choices $\{0,1,10\}$, and $\lambda=0.04$. The unrestricted profile
$(\tau_A,\tau_B)=(1,10)$ gives $H=11$, $C=31$ and utility $9.76$. Among the six
ordered profiles, the best is $(1,0)$, with $H=10$, $C=11$ and utility $9.56$.
The profile $(10,10)$ has utility $11-0.04(40)=9.40$. Exhaustion of the remaining
ordered profiles proves the strict gap.
\end{proof}

Let $h_g(k)$ count local block hits for group $g$ at common-grid label $k$, with
request multiplicity retained, and let $c_g(k)=C_g(q_k)$. Under coherence these
local hits equal the usable hits. For a group $v$, define
\begin{equation}
 D_v(k)=h_v(k)-\lambda c_v(k)
       +\sum_{u:p(u)=v}\max_{0\leq j\leq k}D_u(j).\label{eq:dp}
\end{equation}
The optimum is the sum of $\max_kD_v(k)$ over roots. A postorder traversal with
prefix maxima and traceback takes $O(G(K+1))$ time and storage. This follows by
conditioning on the parent's label and optimizing independent child subtrees.
Smallest maximizing labels resolve ties.

Define $U_{\rm coh}^*$ as the coherent optimum, $U_{\rm PCR}^*$ as the
unrestricted usable-prefix optimum, and the independently optimized local-hit
score
\begin{equation}
 U_{\rm ind}^*=\sum_g\max_k\{h_g(k)-\lambda c_g(k)\}.
\end{equation}
All three quantities in the following bound use the same finite grid and group
assignment.
\begin{corollary}[Certification of a simpler policy]\label{cor:bounds}
\begin{equation}
 U_{\rm coh}^*\leq U_{\rm PCR}^*\leq U_{\rm ind}^*.\label{eq:sandwich}
\end{equation}
Thus $U_{\rm ind}^*-U_{\rm coh}^*$ is a computable upper bound on the price of
coherence. If it vanishes, the coherent policy is unrestricted-optimal.
\end{corollary}
\begin{proof}
The coherent feasible set is a subset of the unrestricted set. On any profile,
$\sum_{r,d}z_{r,d}\leq\sum_{r,d}x_{r,d}$, while costs are identical. Maximizing
the latter separable objective proves the upper bound.
\end{proof}
The raw-hit bound need not be close. If a root is never retained but many deeper
blocks survive until a second request, raw local hits can be positive while
usable-prefix hits are zero. This is an accounting counterexample, not an
empirical claim that the released traces exhibit arbitrary distortion.

\subsection{Expected memory-time budgets}
For any deterministic policy with $C\leq B$ and any price $\lambda\geq0$,
\begin{equation}
 H\leq U_{\rm PCR}^*(\lambda)+\lambda B.\label{eq:budgetbound}
\end{equation}
This is a Lagrangian upper bound. It does not imply that a single priced policy
solves every deterministic budget constraint.

\begin{proposition}[Optimal mixtures for an expected budget]\label{prop:mixtures}
Fix a finite policy class containing no retention. Permit a policy to be drawn
once before replay, independently of the trace, and require $\mathbb E C\leq B$.
The optimal expected reward is the nondecreasing upper boundary of the convex
hull of achievable $(C,H)$ pairs. An optimum uses at most two policies. On a
supported segment,~\eqref{eq:budgetbound} is tight with the corresponding
supporting price.
\end{proposition}
\begin{proof}
A distribution on policies has expected $(C,H)$ equal to a convex combination
of their pairs; every such combination is implementable by a draw before replay.
Maximizing expected reward is a linear program with one normalization equality
and one budget inequality. An extreme optimal solution has at most two positive
probabilities (one if the budget is slack). A supporting line of the efficient
upper hull has nonnegative slope and yields the stated price certificate.
\end{proof}

For adjacent candidate points $(C_a,H_a)$ and $(C_b,H_b)$, query the optimizer at
$\lambda=(H_b-H_a)/(C_b-C_a)$. If its utility exceeds their supporting line, insert
the returned point and recurse; otherwise the segment is certified. Our released
frontier implementation applies this procedure to the coherent dynamic program.
The same argument allows the closure solver as its oracle. Randomizing whole
profiles across independent windows is different from switching timeouts midway
through a live window, which would change the cache state. Neither the mixture
nor a mean-occupancy target guarantees an instantaneous memory limit.

\section{Computational study}\label{sec:exp}
\subsection{Data, endpoint and chronology}
We use the public Mooncake FAST'25 traces \cite{trace,mooncake}: two
production-derived workloads, \texttt{conversation} and \texttt{toolagent}, and
one synthetic workload. Table~\ref{tab:data} gives their sizes. The release is
pinned by Git revision and SHA-256 digests. Every hash path is checked for the
unique-predecessor property. Hashes represent 512-token blocks; a final partial
block is removed because it cannot be counted as a reusable full block.

\begin{table}[t]
\centering\small\setlength{\tabcolsep}{5pt}
\begin{tabular}{lrrrrr}
\toprule
Trace & Requests & \shortstack{Full\\lookups} & \shortstack{Unique\\blocks} & \shortstack{Duration\\(s)} & \shortstack{Train/validation/test\\requests}\\
\midrule
Conversation & 12,031 & 276,491 & 170,899 & 3,537.0 & 4,401 / 2,447 / 5,183\\
Tool agent & 23,608 & 386,041 & 159,983 & 3,537.0 & 8,779 / 4,832 / 9,997\\
Synthetic & 3,993 & 117,888 & 40,148 & 1,022.0 & 1,512 / 807 / 1,674\\
\bottomrule
\end{tabular}
\caption{Public traces after discarding incomplete final blocks. Splits are by
elapsed time, not request count. There are 39,632 requests in total, of which
35,639 are production-derived.}\label{tab:data}
\end{table}

Each trace is split chronologically into 40\% training, 20\% validation and 40\%
test by its last arrival time. Each replay starts empty and stops at its own
window boundary. Timeouts and group keys are fitted on training only. The primary
endpoint is $(H-\lambda C)/M$ on test; secondary measurements are $H/M$, mean
resident blocks $C/T$, peak blocks and fitting time. A paired resampling of
60-second test buckets gives descriptive 95\% intervals for fixed-policy utility
differences (10,000 resamples, seed 20260923). There are only 24 buckets on each
production-derived test trace and seven on the synthetic trace. These intervals
are neither deployment-level replications nor simultaneous confidence bands
across prices, and may understate uncertainty under longer-range dependence.

A local analysis plan was committed before the first policy comparison. It
specified the data, time split, grids, primary coherent-policy selection and
baselines. After inspecting that study, the certificate formulation, timestamp
coalescing and continuous-timeout analyses were added as explicitly exploratory
extensions. They use the same held-out segments, so the paper does not describe
them as a fresh confirmatory evaluation or an external preregistration.

\subsection{Policy classes and baselines}
The fixed timeout grid, in seconds, is
\[
\{0,0.001,0.1,0.3,1,3,10,30,60,120,300,600,900\}.
\]
All policies are evaluated at every price in
\[
\{0,10^{-4},3\!\cdot\!10^{-4},10^{-3},3\!\cdot\!10^{-3},10^{-2},
3\!\cdot\!10^{-2},10^{-1},0.3,1\}.
\]
No single price is chosen as the primary endpoint after seeing the results.
\emph{Global TTL} fits one timeout. \emph{Group TTL} shares one timeout across
first blocks and assigns a timeout to each of the eight most frequent second-block
prefixes and a pooled fallback; child timeouts are bounded by the root timeout.
\emph{Depth TTL} uses one chain with depth bands
\[
[1],\ [2],\ [3,4],\ [5,8],\ [9,16],\ [17,32],\ [33,64],\ [65,128],\ [129,\infty).
\]
\emph{Group-depth TTL} uses the shared root and one eight-band chain per prefix
group, totaling 73 labels when all eight selected prefixes exist. In every case
the root label can cover many distinct physical root blocks. Unseen second-block
hashes use the training-defined fallback.

The primary \emph{selected coherent} policy chooses depth or group-depth using
validation utility at each price; ties favor depth. All profiles remain trained
on training only. PCR is evaluated for each of the four groupings without the
coherence restriction. An independent-local-hit optimizer is retained as an
accounting diagnostic, with actual prefix hits replayed afterward. The
continuous-timeout extension evaluates group-depth PCR at all observed training
reuse ages. It also allows ages above 900 seconds, so its comparison with the
fixed grid combines finer resolution and a larger maximum timeout.

Capacity-constrained least-recently-used (LRU) and first-in-first-out (FIFO)
baselines use 256, 1,024, 4,096 and 16,384 blocks. They have the same batch read
semantics and usable-prefix hit metric, with stable trace order for post-batch
insertions. These are hash-store reference policies, not replicas of a particular
serving engine's radix eviction. Their hard capacity constraint differs from
priced memory-time. A no-eviction replay provides the cold-window reuse reference.

\subsection{Verification and numerical implementation}
The 11-test suite includes the following independent checks:
\begin{itemize}
\item 140 random forests with up to seven nodes and four labels: exhaustive
feasible-profile enumeration versus the dynamic program and a separate NetworkX
minimum-cut encoding.
\item 80 random request traces with tied arrivals: all 64 profiles per trace
checked against event replay, unrestricted cut versus brute force, and a
coherence-constrained cut versus the dynamic program.
\item 20 additional traces: the event-age optimizer compared with exhaustive
integer timeouts from 0 through 10 seconds, including values absent from its grids.
\item 30 random five-node problems, each at 15 budgets: coherent frontier
interpolation versus a linear program over all feasible policies (450 comparisons).
\item Randomized replay accounting and explicit checks for zero timeout,
simultaneous requests, terminal censoring, orphaned blocks and decimal boundaries.
\end{itemize}
For every full-trace fitted profile in the closure studies, an expiration-heap
replay independently recomputes training hits and occupied time and checks the
optimization result. This tests the accounting as well as the solver.

The implementation uses Python 3.12, NumPy and PyMaxflow \cite{pymaxflow}. Flow capacities are
double precision; the proofs establish mathematical exactness, while the code
provides numerical checks rather than an exact-arithmetic certificate. The largest
absolute flow/objective residual among the 120 native fixed-grid cuts is
$1.64\times10^{-11}$ saved-block units. Returned cut labels are reduced to the
smallest thresholds preserving their selected training hits, removing arbitrary
zero-weight retention choices. Machine-specific timings are single observations
on an ARM64 macOS system and exclude data download; they are not a hardware
benchmark.

\section{Results}\label{sec:results}
\subsection{The coherent restriction is nearly exact on these traces}
Across three native traces, four groupings and ten prices, the unrestricted
fixed-grid optimum equals the coherent optimum to numerical tolerance in 118
of 120 cases. The two exceptions are the tool-agent group and group-depth models
at $\lambda=0.01$, each improving training utility by 4.718 saved-block units,
or $3.09\times10^{-5}$ per training input block. The largest corresponding test
utility improvement over the same-grouping coherent policy is
$2.79\times10^{-5}$ per input block. Thus the exact solver predominantly certifies
a simpler policy on these workloads; it does not demonstrate a large empirical
advantage from abandoning coherence.

The independent-local-hit upper bound can be less informative. For synthetic
group-depth retention at $\lambda=0.003$, the exact optimum is 599.160 while the
raw-hit bound is 865.498. At $\lambda=0.03$, the exact optimum is zero while the
raw-hit bound remains 31.636. The bound is valid but can ascribe value to
incompatible local hits. Reporting it together with the attainable coherent
solution prevents that value from being mistaken for reusable-prefix savings.

\subsection{Heterogeneity helps in specific storage-price regimes}
Figure~\ref{fig:utility} reports every selected-policy contrast; the complete
utility table is in Appendix~\ref{app:tables}. Tool-agent utility is greater than
the global-TTL baseline at all nine positive prices. Conversation utility is
lower at $10^{-4}$ and $3\times10^{-4}$, higher at the remaining seven positive
prices, and tied at zero. Synthetic utility is lower at prices from zero through
$10^{-3}$, higher at $0.003$, and tied at the five largest prices. This variation
rules out a blanket claim that more detailed profiles transfer better.

\begin{figure}[t]
\centering\includegraphics[width=\linewidth]{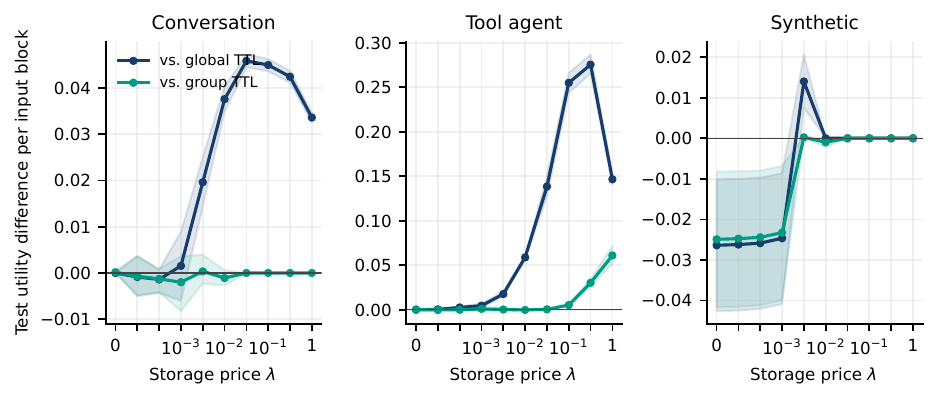}
\caption{Held-out utility difference of the validation-selected coherent policy
against global and group TTL, at all ten declared prices. Shaded regions are
descriptive paired 60-second-bucket bootstrap intervals, not simultaneous bands.
The horizontal positions enumerate the declared prices; the axis is not linear
in price. Negative differences are retained.}\label{fig:utility}
\end{figure}

As an interpretable operating point rather than a selected headline, at
$\lambda=0.01$ on tool-agent test data the selected policy reuses 37.113\% of full
input blocks with 69.85 mean resident blocks. Global TTL reuses 37.023\% with
715.97 mean blocks. The 90.2\% reduction in mean occupancy accompanies a similar
hit fraction in this replay. However, group TTL already achieves 37.139\% with
69.83 mean blocks and has slightly greater utility than the selected policy.
Most of this operating point's improvement therefore comes from prefix grouping,
an established idea, rather than the additional depth structure.

At larger tool-agent prices, depth structure becomes more useful: the selected
policy exceeds group-TTL utility by 0.00521, 0.03023 and 0.06094 per input block
at prices 0.1, 0.3 and 1, respectively. The synthetic trace gives the counterpoint:
at zero price the selected policy's hit fraction is 0.68154, compared with
0.70793 for global TTL. The training-optimal short profiles do not preserve all
reuse patterns appearing later. Validation selection between the two structured
classes does not protect against a global baseline outperforming both.

\begin{figure}[t]
\centering\includegraphics[width=\linewidth]{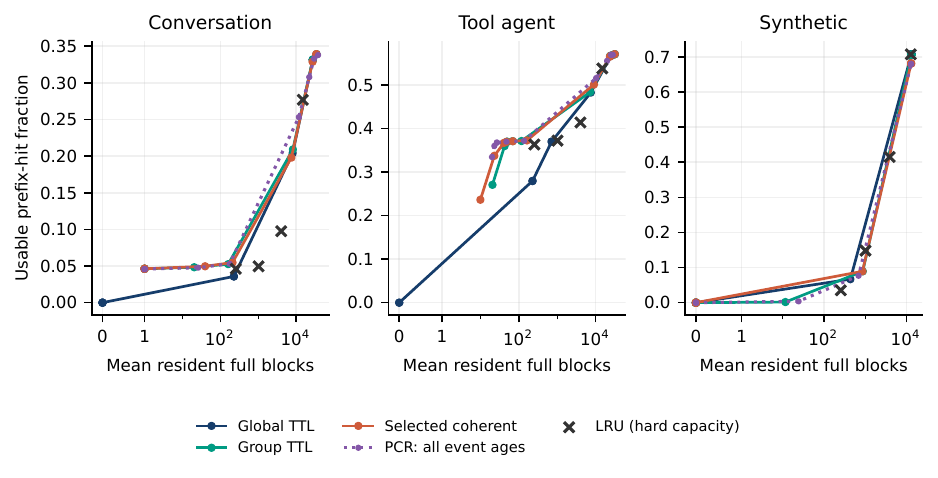}
\caption{Realized test hit fraction versus mean occupancy. TTL points are fitted
on training at the declared prices; connecting them is a visual guide, not a
certified test frontier. PCR with all event ages is exploratory. Crosses show
hard-capacity LRU at four capacities, plotted at its realized mean occupancy;
TTL and LRU do not have equivalent instantaneous memory guarantees. The synthetic
workload exposes weak transfer of structured static profiles.}\label{fig:tradeoffs}
\end{figure}

Figure~\ref{fig:tradeoffs} and Appendix~\ref{app:capacity} provide the
capacity-policy context. For example, tool-agent LRU at capacity 256 attains
0.36355 hit fraction and 255.38 mean blocks. At capacity 16,384 it attains 0.53841,
whereas the no-eviction reference is 0.57223. These comparisons describe the
observed tradeoffs and do not establish dominance at matched hard capacity.

\subsection{Continuous optimality and transfer are different properties}
The event-age construction optimizes all nonnegative static group-depth timeouts
on training. Figure~\ref{fig:event} shows its full comparison with fixed-grid PCR.
Training utility never decreases, as required by Theorem~\ref{thm:event}.
Tool-agent test improvements reach 0.02264 utility per input block at price 0.3,
but the event-age profile is worse at prices zero and 0.003. On the synthetic
trace, its test utility is 0.00925 lower at price 0.003 even though its training
utility improves. Conversation transfer is also mostly weaker. Exact empirical
optimization is therefore an oracle for a specified historical objective, not
a guarantee against workload shift or overfitting.

\begin{figure}[t]
\centering\includegraphics[width=.95\linewidth]{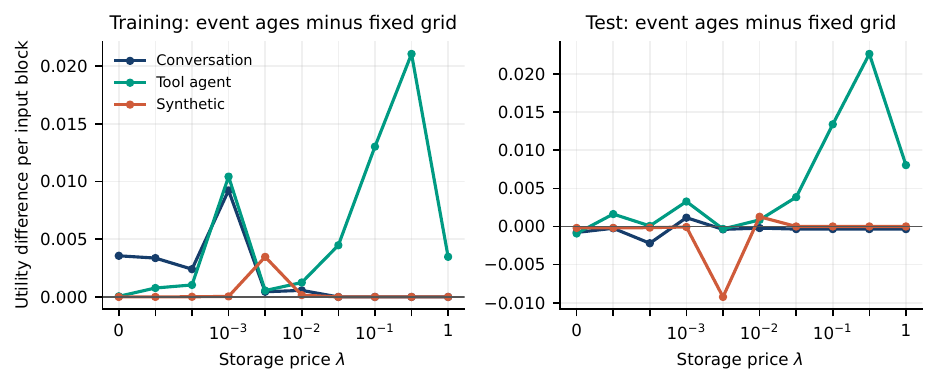}
\caption{Exploratory event-age PCR versus the frozen 13-choice grid, using the
same group-depth class. Training gains include the effect of allowing timeouts
above 900 seconds. Exact fitting does not uniformly improve the later window.}
\label{fig:event}
\end{figure}

\subsection{Computation, budgets and timestamp sensitivity}
Table~\ref{tab:runtime} gives graph sizes and observed solver times. Compression
reduces fixed-grid group-depth graphs to 2,362--4,694 nodes, while full event-age
graphs use 9,995--28,280 nodes. These sizes make the method suitable as an offline
configuration oracle for the released traces. The coherent frontier calculation
finds 59, 66 and 11 supported group-depth profiles on conversation, tool-agent
and synthetic training data, respectively. Their certified mixture frontiers
appear in Figure~\ref{fig:budget}. These are training guarantees only.

\begin{table}[t]
\centering\small\begin{tabular}{lrrrrr}
\toprule
Trace & Grid nodes & Event nodes & Grid solve (ms) & Event solve (ms) & Prep. (s)\\
\midrule
Conversation & 4,694 & 28,280 & 4.25 & 24.41 & 0.62\\
Tool agent & 4,341 & 26,901 & 4.00 & 23.44 & 0.68\\
Synthetic & 2,362 & 9,995 & 1.68 & 8.47 & 0.20\\
\bottomrule
\end{tabular}

\caption{Group-depth graph size and median solve time over ten prices. Preparation
time is for the event-age graph and cost tables. Solver calls use already
prepared inputs; timings are single-run descriptive measurements.}
\label{tab:runtime}
\end{table}

Coalescing arrival times to whole seconds and refitting the same fixed-grid
policies changes test utility by at most 0.000404, 0.009114 and 0.001198 across
all groupings, modes and prices for conversation, tool-agent and synthetic data,
respectively. This checks sensitivity to timestamp resolution and batching. It
does not simulate positive request duration: a block still becomes available
after its batch in the coalesced replay. The maximum tool-agent change occurs
outside the group-depth PCR curve, whose maximum absolute change is 0.001449.

\begin{figure}[t]
\centering\includegraphics[width=\linewidth]{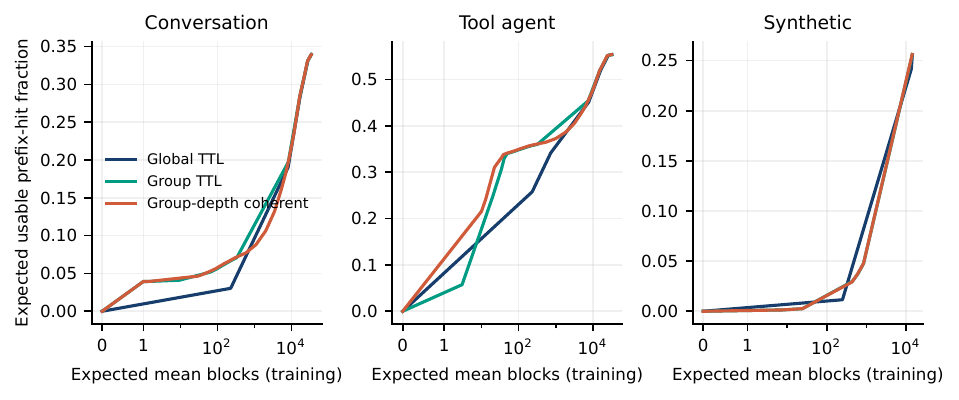}
\caption{Certified training frontiers for mixtures of coherent policies. Lines
interpolate whole-profile randomizations under an expected memory-time budget.
They are not guarantees on peak residency, nor held-out performance.}
\label{fig:budget}
\end{figure}

\section{Industrial use and limitations}\label{sec:limits}
A practical use of PCR is a retention-planning and audit component for an
independently addressed KV store. Given recent access logs, it can return a
timeout profile, its fitted reward and occupied time, the coherent lower bound,
and the local-hit upper bound. If the bounds agree, an operator has a concrete
reason to deploy the simpler coherent profile. If they differ, the exact graph
quantifies the historical opportunity before any engine integration. Repeated
offline fitting and a shadow replay on subsequent windows are plausible next
steps, but their value is not established by the present one-window evaluation.

The unit reward counts saved full prefix blocks. Actual prefill cost depends on
position, model architecture, batching, kernels and transfer overhead. The graph
accepts exogenous nonnegative per-hit rewards, but a reward that changes with the
chosen profile or queue state requires a different model. Similarly, a store
that automatically deletes descendants when an ancestor expires does not follow
our unrestricted residency model. The coherent variant is compatible with the
prefix-residency invariant, subject to the engine's other constraints.

The experiments use short, sampled workloads from one public release, one
chronological split, and a deliberately small family of transferable groups.
There is no claim of optimal online adaptation, generalization to other model
families, or measured dollar savings. Peak occupancy is reported because an
attractive mean can coexist with an infeasible instantaneous requirement. The
graph objective contains no hard-capacity constraint. Its mathematical certificate
also assumes that accesses refresh every referenced full block; selective writes
would require revised accounting.

The next systems experiment should record block completion and transfer times,
active-request pinning, actual peak allocations and model-specific prefill cost,
then compare policies at a common capacity and service-level objective. The
present code establishes a checked optimization component for that experiment.

\section{Conclusion}
Usable prefix reuse is a conjunction of local retention events. Making those
prerequisites explicit turns a seemingly coupled static-timeout problem into a
maximum-weight closure, and observed reuse ages suffice for continuous timeout
optimization. The resulting oracle can both optimize unrestricted retention and
certify a simpler coherent policy. Public-trace results show substantial room
for differentiated retention in some storage-price regimes, little additional
benefit from dropping coherence on the fixed grid, and nonuniform transfer from
finer empirical fitting. These distinctions make the method useful as a precise
configuration model and benchmark without conflating mathematical optimality
with unmeasured system performance.

\section*{Data and software availability}
The repository contains the optimizer, independent replay, exhaustive tests,
pinned trace-download manifest, all fitted profiles and result tables, figure
generation, and manuscript source. Raw traces are obtained from their original
Mooncake release rather than redistributed. The initial local protocol,
exploratory extensions and implementation corrections are recorded separately.

{\small\setlength{\bibsep}{3pt}
\bibliographystyle{plainnat}
\bibliography{references}}

\begin{thebibliography}{13}
\providecommand{\natexlab}[1]{#1}
\providecommand{\url}[1]{\texttt{#1}}
\expandafter\ifx\csname urlstyle\endcsname\relax
  \providecommand{\doi}[1]{doi: #1}\else
  \providecommand{\doi}{doi: \begingroup \urlstyle{rm}\Url}\fi

\bibitem[Dallot et~al.(2024)Dallot, Fesharaki, Pacut, and Schmid]{dependency}
Julien Dallot, Amirmehdi~Jafari Fesharaki, Maciej Pacut, and Stefan Schmid.
\newblock Dependency-aware online caching.
\newblock In \emph{IEEE INFOCOM}, 2024.
\newblock URL \url{https://schmiste.github.io/infocom24caching.pdf}.

\bibitem[Ferragut et~al.(2018)Ferragut, Rodriguez, and Paganini]{ferragut}
Andres Ferragut, Ismael Rodriguez, and Fernando Paganini.
\newblock Optimal timer-based caching policies for general arrival processes.
\newblock \emph{Queueing Systems}, 88:\penalty0 207--241, 2018.
\newblock \doi{10.1007/s11134-017-9540-3}.

\bibitem[Gao et~al.(2024)Gao, He, Sharma, Kang, Jevdjic, Deng, Yang, Yu, and
  Zuo]{cachedattention}
Bin Gao, Zhuomin He, Puru Sharma, Qingxuan Kang, Djordje Jevdjic, Junbo Deng,
  Xingkun Yang, Zhou Yu, and Pengfei Zuo.
\newblock Cost-efficient large language model serving for multi-turn
  conversations with {CachedAttention}.
\newblock In \emph{USENIX Annual Technical Conference (ATC)}, pages 111--126,
  2024.
\newblock URL \url{https://arxiv.org/abs/2403.19708}.

\bibitem[Goseling and Simeone(2018)]{softttl}
Jasper Goseling and Osvaldo Simeone.
\newblock {Soft-TTL}: Time-varying fractional caching.
\newblock arXiv:1807.03537, 2018.
\newblock URL \url{https://arxiv.org/abs/1807.03537}.

\bibitem[Guo et~al.(2026)Guo, Wu, and Yiu]{saga}
Dongxin Guo, Jikun Wu, and Siu~Ming Yiu.
\newblock {SAGA}: Workflow-atomic scheduling for {AI} agent inference on {GPU}
  clusters.
\newblock arXiv:2605.00528, 2026.
\newblock URL \url{https://arxiv.org/abs/2605.00528}.

\bibitem[Kwon et~al.(2023)Kwon, Li, Zhuang, Sheng, Zheng, Yu, Gonzalez, Zhang,
  and Stoica]{paged}
Woosuk Kwon, Zhuohan Li, Siyuan Zhuang, Ying Sheng, Lianmin Zheng, Cody~Hao Yu,
  Joseph~E. Gonzalez, Hao Zhang, and Ion Stoica.
\newblock Efficient memory management for large language model serving with
  {PagedAttention}.
\newblock In \emph{29th ACM Symposium on Operating Systems Principles (SOSP)},
  2023.
\newblock URL \url{https://arxiv.org/abs/2309.06180}.

\bibitem[{Mooncake contributors}(2025)]{trace}
{Mooncake contributors}.
\newblock {Mooncake FAST'25} release: public request traces, 2025.
\newblock URL
  \url{https://github.com/kvcache-ai/Mooncake/tree/621e36ba9eaf400d66d36ea4e56a5928bb8d2a2c/FAST25-release}.
\newblock Revision 621e36ba9eaf400d66d36ea4e56a5928bb8d2a2c; accessed 23
  September 2026.

\bibitem[Neila(2026)]{pymaxflow}
Pablo~M{\'a}rquez Neila.
\newblock {PyMaxflow}: Python library for graph cuts, 2026.
\newblock URL \url{https://github.com/pmneila/PyMaxflow}.
\newblock Software, version 1.3.2.

\bibitem[Picard(1976)]{picard}
Jean-Claude Picard.
\newblock Maximal closure of a graph and applications to combinatorial
  problems.
\newblock \emph{Management Science}, 22\penalty0 (11):\penalty0 1268--1272,
  1976.
\newblock \doi{10.1287/mnsc.22.11.1268}.

\bibitem[Qin et~al.(2025)Qin, Li, He, Cui, Ren, Zhang, Wu, Zheng, and
  Xu]{mooncake}
Ruoyu Qin, Zheming Li, Weiran He, Jialei Cui, Feng Ren, Mingxing Zhang, Yongwei
  Wu, Weimin Zheng, and Xinran Xu.
\newblock {Mooncake}: Trading more storage for less computation---a
  {KVCache}-centric architecture for serving {LLM} chatbot.
\newblock In \emph{23rd USENIX Conference on File and Storage Technologies
  (FAST)}, pages 155--170, 2025.
\newblock URL \url{https://www.usenix.org/conference/fast25/presentation/qin}.

\bibitem[Zhang et~al.(2026)Zhang, Kim, Feng, Du, Liu, Zhong, Ching, Jiang, and
  Hu]{cachescout}
Rui Zhang, Chaeeun Kim, Shaoting Feng, Kuntai Du, Yuhan Liu, Yi~Zhong,
  Cheng-Wei Ching, Junchen Jiang, and Liting Hu.
\newblock Learning agent execution for {KV}-cache management in agentic
  serving.
\newblock arXiv:2608.14624, 2026.
\newblock URL \url{https://arxiv.org/abs/2608.14624}.

\bibitem[Zheng et~al.(2024)Zheng, Yin, Xie, Sun, Huang, Yu, Cao, Kozyrakis,
  Stoica, Gonzalez, Barrett, and Sheng]{sglang}
Lianmin Zheng, Liangsheng Yin, Zhiqiang Xie, Chuyue Sun, Jeff Huang, Cody~Hao
  Yu, Shiyi Cao, Christos Kozyrakis, Ion Stoica, Joseph~E. Gonzalez, Clark
  Barrett, and Ying Sheng.
\newblock {SGLang}: Efficient execution of structured language model programs.
\newblock arXiv:2312.07104, 2024.
\newblock URL \url{https://arxiv.org/abs/2312.07104}.

\bibitem[Zheng et~al.(2026)Zheng, Wang, Ma, Wang, Wang, Chen, Zhang, Pan,
  Huang, Wu, Zhang, Cai, Liu, Ma, Du, Deng, Wu, Zhu, Zhang, and Li]{kareto}
Xianzhe Zheng, Zhengheng Wang, Ruiyan Ma, Rui Wang, Xiyu Wang, Rui Chen, Peng
  Zhang, Sicheng Pan, Zhangheng Huang, Chenxin Wu, Yi~Zhang, Bo~Cai, Kan Liu,
  Teng Ma, Yin Du, Dong Deng, Sai Wu, Guoyun Zhu, Wei Zhang, and Feifei Li.
\newblock Adaptive multi-objective tiered storage configuration for {KV} cache
  in {LLM} service.
\newblock arXiv:2603.08739, 2026.
\newblock URL \url{https://arxiv.org/abs/2603.08739}.

\end{thebibliography}

\clearpage\appendix
\section{Complete declared-price utility table}\label{app:tables}
\noindent Table values are $(H-\lambda C)/M$ on the test window. ``Selected''
denotes validation-selected coherent depth or group-depth retention. ``PCR grid''
is unrestricted group-depth retention on the frozen grid; ``PCR ages'' uses every
training reuse age. Rows are generated directly from the result files.
\begin{center}\small\setlength{\tabcolsep}{4pt}
\begin{tabular}{llrrrrrr}
\toprule
Trace & $\lambda$ & Global & Group & Selected & PCR grid & PCR ages & Sel.\ $-$ Global\\
\midrule
Conversation & 0 & 0.33917 & 0.33901 & 0.33917 & 0.33875 & 0.33796 & 0.00000\\
Conversation & 0.0001 & 0.29703 & 0.29688 & 0.29617 & 0.29578 & 0.29554 & -0.00086\\
Conversation & 0.0003 & 0.22758 & 0.22746 & 0.22619 & 0.22596 & 0.22376 & -0.00140\\
Conversation & 0.001 & 0.10211 & 0.10563 & 0.10365 & 0.10365 & 0.10480 & 0.00154\\
Conversation & 0.003 & 0.02721 & 0.04647 & 0.04682 & 0.04682 & 0.04644 & 0.01961\\
Conversation & 0.01 & 0.00711 & 0.04577 & 0.04468 & 0.04468 & 0.04447 & 0.03757\\
Conversation & 0.03 & 0.00000 & 0.04584 & 0.04584 & 0.04584 & 0.04549 & 0.04584\\
Conversation & 0.1 & 0.00000 & 0.04496 & 0.04496 & 0.04496 & 0.04461 & 0.04496\\
Conversation & 0.3 & 0.00000 & 0.04243 & 0.04243 & 0.04243 & 0.04209 & 0.04243\\
Conversation & 1 & 0.00000 & 0.03360 & 0.03360 & 0.03360 & 0.03326 & 0.03360\\
\midrule
Tool agent & 0 & 0.57121 & 0.57110 & 0.57121 & 0.57091 & 0.56999 & 0.00000\\
Tool agent & 0.0001 & 0.54504 & 0.54573 & 0.54544 & 0.54545 & 0.54707 & 0.00040\\
Tool agent & 0.0003 & 0.49983 & 0.50238 & 0.50236 & 0.50236 & 0.50243 & 0.00253\\
Tool agent & 0.001 & 0.41582 & 0.41946 & 0.42037 & 0.42037 & 0.42364 & 0.00455\\
Tool agent & 0.003 & 0.35097 & 0.36837 & 0.36859 & 0.36859 & 0.36826 & 0.01762\\
Tool agent & 0.01 & 0.30602 & 0.36513 & 0.36486 & 0.36489 & 0.36575 & 0.05884\\
Tool agent & 0.03 & 0.21846 & 0.35645 & 0.35688 & 0.35688 & 0.36073 & 0.13843\\
Tool agent & 0.1 & 0.07546 & 0.32522 & 0.33043 & 0.33043 & 0.34383 & 0.25497\\
Tool agent & 0.3 & 0.00000 & 0.24521 & 0.27544 & 0.27544 & 0.29808 & 0.27544\\
Tool agent & 1 & 0.00000 & 0.08566 & 0.14660 & 0.14660 & 0.15465 & 0.14660\\
\midrule
Synthetic & 0 & 0.70793 & 0.70646 & 0.68154 & 0.68007 & 0.67988 & -0.02639\\
Synthetic & 0.0001 & 0.69913 & 0.69767 & 0.67292 & 0.67146 & 0.67128 & -0.02621\\
Synthetic & 0.0003 & 0.68154 & 0.68010 & 0.65567 & 0.65423 & 0.65407 & -0.02587\\
Synthetic & 0.001 & 0.61995 & 0.61858 & 0.59531 & 0.59393 & 0.59386 & -0.02464\\
Synthetic & 0.003 & 0.05826 & 0.07203 & 0.07226 & 0.07226 & 0.06301 & 0.01401\\
Synthetic & 0.01 & 0.00000 & 0.00097 & 0.00000 & 0.00097 & 0.00224 & 0.00000\\
Synthetic & 0.03 & 0.00000 & 0.00000 & 0.00000 & 0.00000 & 0.00000 & 0.00000\\
Synthetic & 0.1 & 0.00000 & 0.00000 & 0.00000 & 0.00000 & 0.00000 & 0.00000\\
Synthetic & 0.3 & 0.00000 & 0.00000 & 0.00000 & 0.00000 & 0.00000 & 0.00000\\
Synthetic & 1 & 0.00000 & 0.00000 & 0.00000 & 0.00000 & 0.00000 & 0.00000\\
\bottomrule
\end{tabular}

\end{center}

\clearpage
\section{Capacity baselines and unbounded reference}\label{app:capacity}
\noindent Capacity is measured in full blocks. All rows use a cold test-window
start. FIFO and LRU are capacity-limited hash-store baselines. Their mean
occupancy is not a separate capacity guarantee for the TTL methods.
\begin{center}\small
\begin{tabular}{llrrrr}
\toprule
Trace & Policy & Capacity & Hit fraction & Mean blocks & Peak blocks\\
\midrule
Conversation & LRU & 256 & 0.04625 & 255.33 & 256\\
Conversation & FIFO & 256 & 0.04259 & 255.33 & 256\\
Conversation & LRU & 1024 & 0.04957 & 1018.52 & 1024\\
Conversation & FIFO & 1024 & 0.04720 & 1018.52 & 1024\\
Conversation & LRU & 4096 & 0.09753 & 4026.41 & 4096\\
Conversation & FIFO & 4096 & 0.09245 & 4026.41 & 4096\\
Conversation & LRU & 16384 & 0.27686 & 15003.35 & 16384\\
Conversation & FIFO & 16384 & 0.24975 & 15003.35 & 16384\\
Conversation & Unbounded & -- & 0.34339 & 38862.20 & 73542\\
\midrule
Tool agent & LRU & 256 & 0.36355 & 255.38 & 256\\
Tool agent & FIFO & 256 & 0.33270 & 255.38 & 256\\
Tool agent & LRU & 1024 & 0.37253 & 1018.67 & 1024\\
Tool agent & FIFO & 1024 & 0.36279 & 1018.67 & 1024\\
Tool agent & LRU & 4096 & 0.41423 & 4024.92 & 4096\\
Tool agent & FIFO & 4096 & 0.40727 & 4024.92 & 4096\\
Tool agent & LRU & 16384 & 0.53841 & 14909.48 & 16384\\
Tool agent & FIFO & 16384 & 0.51924 & 14909.48 & 16384\\
Tool agent & Unbounded & -- & 0.57223 & 35768.20 & 67335\\
\midrule
Synthetic & LRU & 256 & 0.03526 & 255.69 & 256\\
Synthetic & FIFO & 256 & 0.03662 & 255.69 & 256\\
Synthetic & LRU & 1024 & 0.14841 & 1016.80 & 1024\\
Synthetic & FIFO & 1024 & 0.15305 & 1016.80 & 1024\\
Synthetic & LRU & 4096 & 0.41506 & 3942.27 & 4096\\
Synthetic & FIFO & 4096 & 0.39546 & 3942.27 & 4096\\
Synthetic & LRU & 16384 & 0.70769 & 12824.97 & 16384\\
Synthetic & FIFO & 16384 & 0.69476 & 12824.97 & 16384\\
Synthetic & Unbounded & -- & 0.70793 & 13317.67 & 18074\\
\bottomrule
\end{tabular}

\end{center}

\clearpage
\section{Reproduction and audit details}
The repository's \texttt{Makefile} separates data download, tests, experiments,
figures and typesetting. The source revision and file digests are recorded in
\texttt{data/manifest.json}; experiment outputs include complete price sweeps,
labels, raw and usable hit counts, mean and peak occupancy, and runtime metadata.
The prospective protocol digest is
\begin{center}\footnotesize\texttt{3f47a960f95a3f7b13fbf68dfc8fab51fde3f12a2f9b7422b0a06497b9a31719}.\end{center}
This digest identifies the local plan, not independent registration. Timing
columns can vary between runs; the replay endpoints and fitted profiles are
deterministic for the declared inputs and numerical environment. Test and
validation caches are intentionally cold rather than initialized from training.
No hidden warm-start state is included in the reported savings.

Integer time comparisons avoid disagreements between subtraction-based age
tests and addition-based expiration tests at decimal boundaries. Source times
have millisecond precision; calculations use integer microseconds so fractional
split boundaries remain representable. Event-grid rounding is downward to actual
training ages, which explains why tied training optima can transfer differently
from longer fixed-grid timeouts. All numerical results in this manuscript were
regenerated after these implementation corrections.

\section*{Computational assistance}
An OpenAI Codex generative-AI assistant was used to develop the formulation,
implement and test the algorithms, run the computational experiments, and draft
and revise the manuscript. Reported numerical results were produced by the
accompanying executable code using the cited public traces. The author is
responsible for the manuscript and its conclusions.
\end{document}